\documentclass[12pt]{article}
\usepackage[utf8]{inputenc}
\usepackage[T1]{fontenc}
\usepackage{amsmath,multicol}
\usepackage{amsfonts}
\usepackage{amsthm}
\usepackage[english]{babel}

\usepackage{graphicx}
\usepackage{url}
\usepackage[dvips,final]{epsfig}
\usepackage[psamsfonts]{amssymb}
\usepackage{latexsym}
\usepackage{xcolor}
\usepackage{colortbl}
\usepackage{tikz}

\usetikzlibrary{arrows}

\usepackage{algorithm}
\usepackage{algpseudocode}
\usepackage{cite}
\usepackage{setspace}

\newtheorem{lemma}{Lemma}
\newtheorem{coro}{Corollary}[lemma]

\begin{document}


\title{Point cloud registration: matching a maximal common subset on pointclouds with noise (with 2D implementation)}
\author{Jorge Arce Garro, B.S. (Universidad de Costa Rica)\\David Jiménez López, Ph.D. (Universidad de Costa Rica)}

\maketitle

\abstract{We analyze the problem of determining whether 2 given point clouds in 2D, with any distinct cardinality and any number of outliers, have subsets of the same size that can be matched via a rigid motion. This problem is important, for example, in the application of fingerprint matching with incomplete data. We propose an algorithm that, under assumptions on the noise tolerance, allows to find corresponding subclouds of the maximum possible size. Our procedure optimizes a potential energy function to do so, which was first inspired in the potential energy interaction that occurs between point charges in electrostatics. 


\section{Introduction}

Let $X$, $Y$ be subsets of $\mathbb{R}^2$, with $X$ and $Y$ having $M$ and $N$ different points respectively. We intend to find a rigid motion that aligns a common subset of the maximum possible size between these 2 sets. Also, this alignment may be noisy: that is, matched points may lie within a distance less than a prescribed tolerance of $\delta$. This tolerance may be established by the desired application: it can be defined, for example, as the measurement uncertainty with which the point cloud data was taken.

In order to find this subset, we may first force a match between a point in $X$, denoted $x_{p}$, and a point in $Y$, denoted $y_{q}$. That is, we apply the traslations $X \rightarrow X - x_p$, $Y \rightarrow Y - y_q$ and place both point clouds in the plane $\mathbb{R}^2$, so that now the points $x_p$ and $y_q$ lie on the origin. Henceforth, we will call $x_p$ and $y_q$ the \textbf{pivots} of the translation.

Define $X_p = X - x_p$, $Y_q = Y - y_q$. For notational simplicity, we will keep naming the points in these sets as $x_i$ and $y_j$ respectively. Now, our problem can be formulated as follows: find a rotation matrix $R(\theta) $ so that, for the maximum $K \in \mathbb{N}$ possible, we have:

\begin{equation} \label{eq: problem}
||x_{i_k} - R(\theta) y_{j_k}||\leq \delta \mbox{, for } k = 1, ..., K
\end{equation}

where the $i_k$ are all different from each other, as well as the $j_k$. When this is done, we can consider this process for all possible pairs of pivots $x_p \in X$ and $y_q \in Y$ and extract from it the maximum value of K possible and a maximal subcloud, provided the process is efficient enough for implementation.


In summary, the features of the problem are:

\begin{itemize}
\item Finding an optimal rigid motion between the 2 point clouds, where correspondences between the points are unknown.
\item The match may be noisy, up to a certain tolerance $\delta$.
\item Notice that the 2 clouds may be of different cardinality, and not all points in a cloud need to have a match with points in the other cloud. That is, there may be outliers in both clouds. 
\end{itemize}

\section{Previous work}

\subsection{Simplest form of the problem: the pointclouds are labeled}

In the simplest form of the pointcloud registration problem, 2 pointclouds $X$ and $Y$ with the same cardinality are given, and "labels" which estabilish a correspondence between the points in both pointclouds are known beforehand. The goal is to find a rigid motion that can be applied to the first cloud to transform it as closely as possible into the other, all while respecting the correspondences. This is called the "Procrustes matching problem" and methods of solution have been given since at least 1952 \cite{gower2004procrustes}. Among these, we can find the use of "cost functions" or "energy functions", as well as the use of principal component analysis (PCA). These methods can even deal with noise on the pointcloud, since the match need not be exact.

The existence of labels can be exploited even further by the "statistical theory of shape" \cite{kendall1989survey}. This theory uses methods of differential geometry to describe the "shape" of a labeled pointcloud; that is, the information of the pointcloud that is invariant to translations, rotations, and rescaling. The information is characterized as a point in a differential manifold, and by providing a metric to the manifold, comparison methods between shapes are obtained.

However, there are several applications where data is unlabeled, and the correspondence between points in $X$ and $Y$ is not known a priori. One of the most popular tools to deal with this situation is the ICP (iterative closest point) algorithm \cite{bellekens2015benchmark}. From an initial positioning of both pointclouds (or initialization), a label correspondence is estabilished via a nearest-neighbor search (a "linear assignment" problem), and then these labels are used to compute the rigid motion that best aligns the labeled points, usually via the optimization of an energy function (this step is usually called a "least squares" problem). Next, the new position is taken as the initial position, and the algorithm is iterated as many times as required to get a satisfactory match. There are even some variants of ICP capable of discarding outliers in the labeling step. We will discuss more about ICP variants and their use of energy functions soon.

While it is effective in many cases and simple, the iterative nature of ICP makes it susceptible to converge to a false solution if the initialization is too far away from the actual alignment, or to never converge at all. An alternative option is seen in the work in \cite{jimenez2013matching}, which manages to extend PCA analysis in order to compute both the rigid motion and the label assignment at once. This method cannot compare matrices of different size, however.

\subsection{Working without labels}

Although the methods described so far can deal with noisy matches, the necessity of labels (or of computing labels) can be restrictive to study variations of the problem where there is no bijective correspondence between both pointclouds. This can happen in several ways: any of the pointclouds may have outliers, or there may be an oversampling of one of the pointclouds (that is, both pointclouds represent the same object but one of them uses more points to do so). One may also be interested in finding a partial match, that is, finding subsets of the maximum possible size for both pointclouds such that these subsets can be identified as approximately the same pointcloud via a rigid motion. In fact, the interest in creating a new method to deal with partial matches is what inspired the writing of this article.

In order to tackle these scenarios, different approaches have been created that do not need to use labels to work with the pointclouds, and instead recast the pointcloud data into new abstract objects which are independent of labeling. For example, a method called ANSIG \cite{rodrigues2010ansig} (which stands for analytical signature) can associate a 2D pointcloud with a complex variable function which is invariant to translations and rescaling, and deals with rotations appropiately. These functions can then be compared, via different function metrics, for pointcloud registration. ANSIG is even robust to oversampling in the pointclouds, but it cannot deal with a partial match. Similarly, one can find Gaussian Mixture methods, such as \cite{jian2005robust} \cite{li2009global}, which associate pointclouds with probability distributions. Gaussian Mixture methods can deal with outliers and noise. Finally, the method in \cite{goodrich1999approximate} uses a Haussdorf distance in order to find partial matches where one of the pointclouds can be seen as a subset of the other.

It has been interesting to observe an interplay of some of these methods with theoretical physics. For example, the Schrödinger distance transform method \cite{deng2014riemannian} constructs a wave function that describes the pointcloud data. There is also a gravitational approach \cite{golyanik2016gravitational} that treats the pointclouds as rigid bodies that attract themselves to each other, and proposes the equilibrium of this physical situation as a matching criterion. This intuition is in fact similar to the one that we will propose.

Finally, it is important to note that the problem of finding a partial match between 2 pointclouds has seen an increase in its study during the last years. The problem is now described as "inlier set maximization" or "consensus set maximization" \cite{bazin2014globally} \cite{lian2017concave}. As an example, the work in \cite{chin2015efficient} uses this terminology and describes an efficient solution method which uses tree search techniques. 

\subsection{Issues of current energy methods for partial matches}

Since our work will optimize an energy function, we will highlight some key differences of our method with existing energy methods. As was seen before, these are well established in the literature and are usually variants of the ICP method. They minimize sums of the form \cite{bellekens2015benchmark} \cite{lian2017concave}

\begin{equation} \label{eq: previous}
\sum_{i} \sum_{j} w_{ij} ||x_i-R(\theta)y_j - t||^2
\end{equation}

Here, $R(\theta)$ and $t$ represent rotation and translation parameters, respectively. Also, the coefficient $w_{ij}$ equals 1 if the points $x_i$ and $y_j$ are thought to be correspondent to each other, and 0 otherwise. Further sofistications have been added to (\ref{eq: problem}), such as doing a "soft assign", which allows the coefficients to be numbers in [0,1] in order to describe matches with uncertainty \cite{rangarajan1997softassign}. It is also common to add a "penalty term" which grows large if the pointcloud registration possess an undesirable feature. \cite{scherzer2015handbook}

The possibility to set the $w_{ij}$ equal to zero is what allows these methods to ignore outliers and their effect on the energy function. However, it is precisely the calculation of these coefficients what leads to a mixture of a linear assignment and a least squares problem that needs to be computed and updated in every step of an iterative algorithm, which may have the issue of converging to a non-global minimum.

\section{Energy function to optimize}

\subsection{An electrostatic analogy}

In order to find a non-iterative approach to point cloud registration, we can propose an energy method based on a physical analogy. Suppose that both point clouds represent rigid bodies (with the points connected by rigid, massless rods, for example). Also, suppose that points in the cloud X are assigned a fixed positive electrical charge and points in cloud Y are assigned a fixed negative charge. Now, we choose a point $x_p \in X$ and a point $y_q \in Y$, and translate the pointclouds so that these 2 points are fixed to the origin. Finally, fix the pointcloud X to the plane and allow the pointcloud Y to rotate freely about $y_q$. Let us call $x_p$ and $y_q$ the \textbf{pivots} of the translation.


The electrostatic forces cause this system to have equilibrium points. These can be found by minimizing the following potential energy in terms of $\theta$, a parameter that describes the rotation of $Y$

\begin{equation} \label{eq: energy}
E_{p,q}(\theta) = \sum_{i} \sum_{j} \phi\left(||x_i-R(\theta)y_j||\right)
\end{equation}

Here $R(\theta)$ is a rotation matrix, and the subindices $p$ and $q$ keep track of the translations $x_i \to x_i - x_p$, $y_j \to y_j -y_q$ made. For notational simplicity, we shall not rename the $x_i$ and $y_j$ after the translation. Also, so far, $\phi$ represents the electrostatic potential energy for 2 point charges of magnitude $q_1$ and $q_2$, separated by a distance of $r$:

\begin{equation} \label{eq: electrostatic}
\phi(r) = \frac{C q_1 q_2}{r^2}
\end{equation}
where $C$ is a physical constant. Note that, to form the potential energy \eqref{eq: energy}, we need not consider summands of the form $\phi\left(||x_i-x_{i'}||\right)$ over pairs of points in the same cloud, because each point cloud is being considered as rigid and these terms will remain constant if a rigid motion is applied to one of the point clouds.

Based on physical intuition, we now explore the possibility that, for each pair $(x_p, y_q)$ of pivots, the equilibrium point corresponding to the global minimum of $\eqref{eq: energy}$ leads to a rotation where as many opposite signed charges as possible overlap. If this were true, studying these configurations for each possible $(x_p, y_q)$ would lead to the desired maximal common subset. 

However, there would be a problem working with the function defined in \eqref{eq: electrostatic}: it is undefined for a distance $r=0$, which represents the desirable case of two points overlapping. Furthermore, its unbounded nature would give rise to numerical inestabilites and the inexistence of a global minimum. It is therefore necessary to find a different $\phi$.

\subsection{Using a step function to handle outliers}

In computer vision, a function $\phi: \mathbb{R}^+ \to \mathbb{R}^+$ as the one used in \eqref{eq: energy} is called a \textbf{potential function}, following \cite{scherzer2015handbook} (p. 159). Despite the existence of this terminology, the idea of using potential functions was not found in the other energy methods that were researched.

Now, in order to find an appropiate $\phi$ for the problem, let us review the features that we desire it to have:

\begin{itemize}

\item $\phi$ should be bounded, so that its numerical analysis becomes simpler and that a global minimum for \eqref{eq: energy} is guaranteed to exist.
\item Since we are not using coefficients $w_{i,j}$ to disregard non-correspondences as in ICP, it is crucial that our $\phi$ manages to modulate the input of outliers and non-correspondences to the energy function. If, for example, one simply takes $\phi(r)=r^2$ as in \eqref{eq: previous}, then the existence of outliers will be excessively punished. In fact, any function that diverges to infinity will add increasing penalties to the energy the further non-correspondent points are from each other.

Instead, we desire a function that rejects 2 non-corresponding points as such when a threshold distance between them is reached, without adding further penalty as the distance increases. This behavior can be enforced by having $\phi(r)$ attain a constant value for large values of $r$. Let us set the constant as 1.



\item Finally, $\phi(r)$ should be small for small distances $r$, so that points that are close to each other give the "reward" of making the energy function \eqref{eq: energy} smaller. In fact, in order to deal with noise later on, we want $\phi$ to map small distances to exactly zero, in order to reward as much as possible not only exact matches, but also matches where points are within our tolerance to noise.

\end{itemize}

A simple way to obtain these 3 features is via a Heaviside step function. Consider $\delta \in \mathbb{R}$ as a parameter that indicates the noise tolerance. (Experimentally, it could be tuned in terms of the uncertainty of measurement present in the pointcloud data). Now define:

\begin{equation} \label{eq: step}
  \phi(r) = \left.
  \begin{cases}
    0 & \text{ for } 0 \leq r< \delta \\
    1 & \text{for } \delta \leq r
  \end{cases}
  \right\} = H(r-\delta)
\end{equation}

This function will cause that any pair of points from $X$ and $Y$ that are closer than $\delta$ are seen as correspondent to each other, by having the corresponding summand in \eqref{eq: energy} not increase the value of the energy function. All other cases are rejected as non-correspondent, and increase the value of \eqref{eq: energy}. 



\subsection{The global minimum of $E_{p,q}$ and the maximal common subset}

Now, before considering the new energy function \eqref{eq: energy} obtained by using the step function above, we must first deal with a potential problem. Among the pairs of points $(x_i,R(\theta)y_j)$ that lay closer than $\delta$ there might be 2 or more pairs with the same point $x_i$ for different values of $j$, or viceversa. When interpreting the results of evaluating the energy function, this would have the effect of considering that one point in a pointcloud is matched with several points in the other one. This is an undesirable behavior, which would not yield the kind of "maximal common subset" we are looking for, stated in \eqref{eq: problem}.

Intuitively speaking, this will tend to happen if points in a pointcloud are too close to be able to tell them apart with a measuring tolerance of $\delta$. Therefore, $\delta$ must be small enough in order for our procedures to work. The next lemma determines exactly how small:

\begin{lemma} \label{condition}

Define, for a pointcloud $X$, the minimum distance between its points:
 
$$\Delta(X) := \min_{1 \leq i < i' \leq M} ||x_i-x_{i'}||$$
Consider now: $\Delta = \min\left\lbrace \Delta(X), \Delta(Y) \right\rbrace$. If we choose $\delta < \Delta /2$, then for any $\theta$ and $x_i \in X$ there is at most one point $y_j \in Y$ such that: $||x_i-R(\theta)y_j|| \leq \delta$, and viceversa. 

\end{lemma}

\begin{proof}

Indeed, assume there are 2 points $y_j$ and $y_{j'}$ such that:

$$||x_i-R(\theta)y_j|| \leq \delta \mbox{ and } ||x_i-R(\theta)y_{j'}|| \leq \delta$$
Reversing the terms in one of the norms and adding yields:

$$||x_i-R(\theta)y_j|| + ||R(\theta)y_{j'}- x_i|| \leq 2 \delta$$
$$\Rightarrow ||x_i-R(\theta)y_j + R(\theta)y_{j'}- x_i|| = ||R(\theta)(y_j - y_{j'})|| = ||y_j - y_{j'}|| \leq 2 \delta$$
applying the triangular inequality and then the fact that a rotation operator preserves norms. By hypothesis, we can see that this implies:
$$||y_j - y_{j'}|| < \Delta$$
a contradiction. An analogous argument can be done for each $y_j$ and two different $x_i$ and $x_{i'}$ .


\end{proof}

Before continuing, there are some terms in the double sum in \eqref{eq: energy} that can be neglected. Notice that under the assumption of Lemma \ref{condition}, the terms corresponding to $i=p$ in the double sum above satisfy $\phi\left(||x_i-R(\theta)y_j||\right) = 1$ for $j=q$ and $0$ otherwise. This is because setting $x_p$ and $y_q$ as the pivots sends them to the origin and forces them to be matched. Thanks to Lemma \ref{condition}, no other point $y_j$ can be matched to $x_p$. There is an analogous behavior for the pairs corresponding to $j=q$.

Therefore:

$$E_{p,q} = \sum_{i} \sum_{j} \phi\left(||x_i-R(\theta)y_j||\right)$$
$$ = \sum_{i\neq p} \sum_{j\neq q} \phi\left(||x_i-R(\theta)y_j||\right) + \sum_{i\neq p} \phi\left(||x_i-R(\theta)y_q||\right) + \sum_{j\neq q} \phi\left(||x_p-R(\theta)y_j||\right) + \phi\left(||x_p-R(\theta)y_q||\right)$$
$$ = \sum_{i\neq p} \sum_{j\neq q} \phi\left(||x_i-R(\theta)y_j||\right) + (M-1) + (N-1) + 0$$

This new formula will allows us to avoid considering the pivots everytime. Let us ignore the constants summed and rename the remaining sum as our new energy function. We will use $\phi$ as defined in \eqref{eq: step} from now on.

\begin{equation} \label{eq: newenergy}
E_{p,q}= \sum_{i\neq p} \sum_{j\neq q} \phi\left(||x_i-R(\theta)y_j||\right)
\end{equation}

Next, let us relate this function with our matching problem.

\begin{lemma} \label{E and K}

Let $K_{p,q}(\theta)$ be the number of pairs $(x_i,R(\theta)y_j)$ with $i \neq p$ and $j \neq q$ that, using $x_p$ and $y_q$ as pivots, are closer than $\delta$ from each other. That is, $K_{p,q}(\theta)$ is the size of the subset being matched with the  pivots $x_p$, $y_q$ and rigid motion $R(\theta)$, minus one. Then:

$$E_{p,q}(\theta) = (M-1)(N-1) - K_{p,q}(\theta)$$

\end{lemma}

\begin{proof}
There are $(M-1)(N-1)$ summands in the formula, and each summand $\phi\left(||x_i-R(\theta)y_j||\right)$ will be $1$ always, except when the aforementioned condition holds, in which case it will be $0$. Since $i \neq p$ and $j \neq q$, there are $K_{p,q}(\theta)$ summands with a value of $0$ by definition, and the result holds.
\end{proof}

\begin{coro} The global minimum for the energy function \eqref{eq: newenergy} is achieved if and only if $R(\theta)$ is a rigid motion, with $x_p$ and $y_q$ as pivots, that makes as many pairs of points in $X$ and $Y$ as possible be matched. 

\begin{proof} Since $E_{p,q}(\theta) = (M-1)(N-1) - K_{p,q}(\theta)$, minimizing $E_{p,q}$ is equivalent to maximizing $K_{p,q}$, and the result follows.
\end{proof}

\end{coro}

We know that if we take $\delta$ as in Lemma \ref{condition}, it is guaranteed that the pairs of points will correspond to distinct points in $X$ and $Y$ for each pair. We close this section by showing that this condition is also necessary for the energy function in \eqref{eq: newenergy} to yield an optimal pointcloud configuration in its minimum. That is, if $\delta> \Delta /2$, there are point cloud configurations where the rigid motion of maximum overlap isn't reached at the global minimum of the energy function. There is a simple counterexample for $M=3$, $N=2$ that we now describe. Consider the cloud $X = \left\lbrace x_1, x_2, x_3 \right\rbrace$ as an equilateral triangle of side $l=2 \sqrt{3}$. The cloud $Y=\left\lbrace y_1,y_2 \right\rbrace$ is a segment of the same length as the side of the triangle. It is clear that the maximal common subsets are obtained when the segment $Y$ is placed on top on a segment of the triangle $X$, such that the points align. 

However, for a big enough $\delta$, this configuration does not yield the global minimum of the energy function (\ref{eq: energy}). For example, consider $\delta = 2$, which violates the condition stated in the past lemma, since $\Delta/2 = l/2 = \sqrt{3} < 2$. In Figure \ref{fig:counterexample}, choose a point in $X$ and a point in $Y$ as pivots. Due to symmetry, one can take $x_3$ and $y_1$ as the pivots without loss of generality. Next, rotate $Y$. Circles of radius $2$ centered in $x_i$ for $i=2,3$ are also sketched, and if $y_2$ is in the interior of any of these circles, then the summand $\phi\left(||x_i-R(\theta)y_2||\right)$ will be zero for the corresponding $i$, since the distance between $x_i$ and $y_2$ will be lesser than $\delta$.

\begin{figure}
    \centering
    \includegraphics[height=5cm]{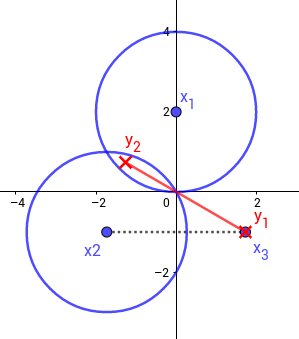}
    \caption{Counterexample illustrating that the restriction in Lemma \ref{condition} is necessary}
    \label{fig:counterexample}
\end{figure}


With the above parameters, it is easy to see that the 2 circles have at least a point in its intersection that is $l$ units away from the pivot $y_1$. Thus, a rotation $R(\theta_0)$ exists such that $y_2$ is mapped to this common region, as shown in the picture. With this configuration, we would have (counting, for each non-pivot point in $Y$, how many corresponding points are there in $X$ that are closer than $\delta$):

$$E_{3,1}(\theta_0) = (2-1) \times (3-1) - 2 = 0$$ 
whereas the value of $E_{3,1}(\theta)$ for the rigid motion of maximum overlap is $(2-1) \times (3-1)-1 =1$. This is the counterexample we sought.

\section{Implementation}

Now we simply have to find the global minimum of $E_{p,q}(\theta)$ for $\theta \in [0, 2 \pi]$, for each $p$ and $q$, and then take the minimum value among these. However, the structure and nature of the energy functions $E_{p,q}$ lead to several pitfalls in a naive implementation.

First of all, the $E_{p,q}$ are not convex functions, as it is normally arranged for in computer vision methods to easily yield the maximum via convex optimization (such as in \cite{lian2017concave}). In fact, they are piecewise constant function for all $p$ and $q$, a behavior that makes derivative-based methods useless. As such, of the generic numerical methods for optimization, only global optimizers (such as simulated annealing) are available. These require a very large number of evaluations of the function to get closer to the maximum, and the $E_{p,q}$ are very expensive computationally, being a sum of $(M-1)(N-1)$ functions.

To give an idea of the order of magnitude of a naive implementation's runtime, in its beginning stages the algorithm was tested simply by uniformly partitioning $[0,2 \pi]$ in $K$ points, and evaluating the function in these points in order to manually approximate the minimum. Each evaluation of $E_{p,q}$ runs through $(M-1)(N-1)$ function evaluations, which then has to be repeated $K$ times and $MN$ times to account for all pivots. This gives a complexity of $\mathcal{O}\left(K(M-1)(N-1)MN\right) = \mathcal{O}(KN^4)$, if we assume $\max(M,N) = N$ without loss of generality.

Executing this for $K=100$, with pointclouds of size $M=N=150$, yielded a runtime of approximately a day. Using a global optimizator would probably require a similar or even higher number of function evaluations. Vectorizing the code in MATLAB and using parallel computing (with 4 cores) led to a new runtime of three minutes for the parameters above, which is still far away from being competitive with other methods.

Without generic numerical methods available, insight on the specific structure of the function is required to optimize it efficiently. It is also necessary if an exact solution is to be found.

\subsection{Summands of $E_{p,q}(\theta)$ seen as indicator functions}

An insight on $E_{p,q}$ that allows for its optimization comes from the fact that $E_{p,q}$ is the sum of $(M-1)(N-1)$ step functions, which can be viewed as characteristic (or indicator) functions of certain subsets of $[0, 2 \pi]$. Once these sets are known, their complements yield the sets where each step function takes the value zero. Finally, the values of $\theta$ where a maximum number of these intervals overlap are exactly the values where $E_{p,q}$ reaches its global minimum.

We refer back to \eqref{eq: newenergy} and fix a pair of indices $p,q$ for the pivots in order to simplify the notation. Define $\phi_{i,j}(\theta) = \phi(||x_i - R(\theta)y_j||)$ for all $i\neq p,j\neq q$, and call $S_{i,j}$ the subset of $[0, 2 \pi]$ where $\phi_{i,j}(\theta) = 0$. This is the complement of the set of which $\phi_{i,j}$ is the characteristic function. Let us calculate it:

\begin{equation} \label{eq: indicator}
\phi_{i,j}(\theta) = 0 \Leftrightarrow \phi(||x_i - R(\theta)y_j||) = 0 \Leftrightarrow ||x_i - R(\theta)y_j|| < \delta
\end{equation}

This last condition can be studied geometrically, as is described in figure \ref{fig:ComputingTheta}. We can also solve for $\theta$ algebraically, by squaring both sides of the inequality and using properties of the norm squared:

$$\Leftrightarrow||x_i||^2 - 2 x_i \cdot R(\theta)y_j + ||y_j||^2 < \delta^2$$
$$\Leftrightarrow||x_i||^2 + ||y_j||^2 - \delta^2 <  2 x_i \cdot R(\theta)y_j$$

Since $i\neq p$ and $j \neq q$, we have $x_i \neq 0$ and $y_j \neq 0$, because all points in each pointcloud are assumed as distinct, and so only $x_p$ and $y_q$ are $0$ after the translation. We divide by $||x_i||\cdot||y_j||$:

\begin{equation}\label{eq: inequality1}
\Leftrightarrow \frac{||x_i||^2 + ||y_j||^2 - \delta^2}{2||x_i||\cdot||y_j||} < \frac{x_i}{||x_i||} \cdot R(\theta) \frac{y_j}{||y_j||} = \hat{x_i} \cdot R(\theta) \hat{y_j} 
\end{equation}
where a hat over a vector denotes the respective unit vector.

Now notice that the right hand side, being the standard inner product of 2 unit vectors (since $R(\theta)$ is norm preserving) simply equals the cosine of the angle between them. In 2 dimensions, this allows for a great simplification, since the angle between the vectors can be obtained as the absolute value of the difference of the polar angles. Let us denote the polar angle of the 2D vector $v$ by $\angle(v)$:

$$\hat{x_i} \cdot R(\theta) \hat{y_j} = \cos\left(|\angle x_i- \angle\left(R(\theta)\hat{y_j}\right)|\right) = \cos (\angle x_i- (\theta + \angle y_j))$$

Here we removed the absolute value in virtue of the parity of $\cos$, and used the fact that $R(\theta)$ increases the polar angle by $\theta$. Replacing this in \eqref{eq: inequality1}, we have:

$$ \frac{||x_i||^2 + ||y_j||^2 - \delta^2}{2||x_i||\cdot||y_j||} < \cos \left(\angle x_i- \theta - \angle y_j\right)$$

Next, we want to take the arccosine. If the left hand side is greater than $1$ or lesser than $-1$, then the inequality above is satisfied for no $\theta$ or for all $\theta$, respectively. Therefore we would have $S_{i,j}$ equal to $\emptyset$ or to $[0, 2\pi]$. Since these sets will intersect with none of the other intervals or with all of them at any $\theta$, we can ignore these cases of $i,j$ without affecting the problem of finding the angles on which the maximum number of $S_{i,j}$ intersect. (Although note that by doing this, we are losing track of the exact value of $K_{p,q}(\theta)$.)

Regardless, for the valid $i$ and $j$, define the angles $\alpha_{i,j} = \arccos \left( \frac{||x_i||^2 + ||y_j||^2 - \delta^2}{2||x_i||\cdot||y_j||} \right)$. The last inequality is equivalent to: 

$$ \angle x_i- \theta - \angle y_j \in [-\alpha_{i,j}, \alpha_{i,j}] $$
$$\Leftrightarrow \theta \in \left[\angle x_i- \angle y_j- \alpha_{i,j}\mbox{ },\angle x_i- \angle y_j +  \alpha_{i,j}\right] := I_{i,j}$$

Considering these intervals $I_{i,j}$ modulo $2 \pi$, we get the sets $S_{i,j}$ we sought. Notice that taking this modulo will yield an interval or the union of 2 non-overlapping intervals. The latter case will happen if the interval calculated above has length lesser than $2 \pi$ and has $0$ as an interior point; if this happens, in order to represent it with angles from $0$ to $2 \pi$, 2 intervals are needed: one that begins in $0$ and another one that ends in $2 \pi$.

\begin{figure}
    \centering
    \includegraphics[height=5cm]{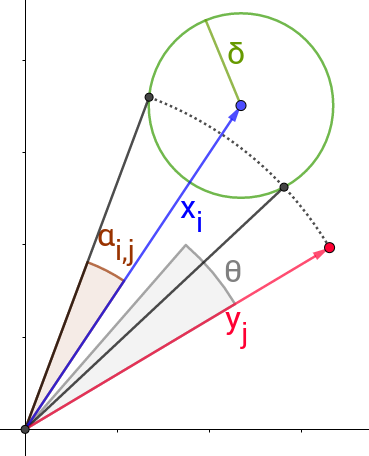}
    \caption{Geometric interpretation of the condition in \eqref{eq: indicator}. We seek the angles of rotation $\theta$ that make $y_j$ lie in a circle of center $x_i$ and radius $\delta$. These can be found by considering the points where $y_j$ would intersect the boundary of the circle as it rotates about the origin, and using them to create the cone shown in the picture. Its semiaperture, easily seen to be $\alpha_{i,j} = \arccos \left( \frac{||x_i||^2 + ||y_j||^2 - \delta^2}{2||x_i||\cdot||y_j||} \right)$ by using the law of cosines, naturally appears in the calculations above. }
    \label{fig:ComputingTheta}
\end{figure}

\subsection{Counting intersections of intervals modulo $2 \pi$ efficiently}


We have calculated the sets $I_{i,j}$ and $S_{i,j}$. Our next problem is: given these families of angle sets, find an angle where the maximum number of intervals overlap. There is a simple way to do this, which we describe below. Let $I_{i,j} = [a_{i,j}, b_{i,j}]$, and for $\theta \in [0, 2\pi]$ let $K(\theta)$ denote the number of $S_{i,j}$ that contain $\theta$. We are deliberately using the same notation as in lemma \ref{E and K}, since the condition $\theta \in S_{i,j}$ implies a noisy match between $x_i$ and $R(\theta) y_j$.

We have that $K(0)$ is the number of sets $S_{i,j}$ that contain the angle $0$. This can be calculated easily, but for the purposes of finding the angles that maximize $K$, we will see it doesn't matter. Letting $\theta$ increase, the value of $K(\theta)$ will remain constant until the variable enters or exits one of the sets $S_{i,j}$.

Keeping control of when this happens is much easier with the intervals $I_{i,j}$: if $\theta$ crosses the initial point modulo $2 \pi$ of any of the $I_{i,j}$, then $\theta$ now lies inside of one of the $S_{i,j}$, by definition, and its value will increase by 1. Similarly, if it crosses the endpoint modulo $2 \pi$ of any of the $I_{i,j}$, $\theta$ no longer belongs to one of the $S_{i,j}$ and its value will decrease by 1. It will remain constant otherwise. In particular, since there are a finite number of these initial and final points, $K(\theta)$ is a piecewise constant function.

Therefore, to maximize $K(\theta)$, one can take all of the angles $a_{ij}$ and $b_{ij}$ modulo $2 \pi$, sort them in increasing order, and then scan the resulting list: adding 1 to a counter for each $b_{ij}$ and substracting 1 for each $a_{ij}$. If any of the elements in the resulting list are equal, then do the substractions first and then the additions, so that the counter doesn't take a large transitory value that doesn't truly reflect the number of intersections. Under this scheme, the angle where the counter reaches its maximum value reveals the initial point of an interval where $K(\theta)$ attains its maximum value. We know the maximum of $K$ is reached at a whole interval because $K$ is a step function. The end point of said interval will be the next number in the sorted list (which must correspond to a $b_{ij}$ that will substract 1 from the counter, or else we wouldn't have had achieved a maximum for the counter).

\begin{algorithm}
\caption{Finding an angle where the maximum number of intersections occur, for a family of angle intervals}
\begin{algorithmic}[1]

\Require
Family of angle intervals $I_{i,j}=[a_{ij},b_{ij}]$

\Ensure
$\theta$: An angle that belongs to the maximum number of these $I_{i,j}$ modulo $2 \pi$.

\State Flatten the array $a_{ij} \mod 2 \pi$ into column A
\State Flatten the array $b_{ij} \mod 2 \pi$ into column B
\State Add a second column to $A$, with all entries equal to 1
\State Add a second column to $B$, with all entries equal to -1
\State Set $T$ equal to the 2-column matrix obtained by joining $A$ and $B$ vertically
\State Quicksort the rows of $T$ lexicographically
\State Set column vector $V$ with entries equal to the partial sums of the second column of $T$ (cummulative sum)
\State $\theta \gets T \left( \arg \max_k {V_k}, 1 \right)$

\end{algorithmic}
\end{algorithm}

Sorting lexicographically will ensure that $-1$ entries are higher in the list in the case of a tie, and thus substractions from the interval counter will be performed first. Again, note that the maximum entry of the vector $V$ is not necessarily $K_{p,q}(\theta)$, since we had to ignore certain cases of $i,j$ where the arccosine in $\alpha_{i,j}$ was not computable, and also because we are not considering the initial value $K(0)$ in our counting.

\subsection{Algorithm to find a maximal common subcloud}

We can finally state the algorithm that will take 2 pointclouds $X$ and $Y$ and find an optimal placing, in order to align a maximal common subcloud. In fact, if one were to keep track of all the angles where the counter mentioned above reaches its maximum value, one could find all maximal common subclouds. The implementation would be similar.

\begin{algorithm}
\caption{Finding the optimal placing of $X$ and $Y$}
\begin{algorithmic}[1]
\Require
\Statex $X,Y$: Pointclouds that we desire to match, encoded as $M \times 2$ and $N \times 2$ matrices respectively
\Statex $M,N$: Number of points in $X$ and $Y$ respectively
\Statex $\delta$: Tolerance for a noisy match, with $\delta \leq \frac{\Delta}{2}$

\Ensure
\Statex $P,Q$: Indices of the pivots that yield an optimal rigid motion
\Statex $\Theta$: Given the pivots $P$ and $Q$, polar angle that $Y$ requires to rotate in order to match as many points as possible with $X$. 

\For{$p$ from 1 to $M$}
    \State Translate $X$ so that $x_p \to$ origin
	\For{$q$ from 1 to $N$}
    	\State Translate $Y$ so that $y_q \to$ origin
	        \State$\alpha_{i,j} \gets \arccos\left(\frac{||x_i||^2 + ||y_j||^2 -\delta^2}{2||x_i||\cdot||y_j||} \right)$, for each valid $i,j$
	        \State $I_{i,j} \gets [\arg(x_i) - \arg(y_j)-\alpha_{i,j}, \arg(x_i) - \arg(y_j) + \alpha_{i,j}]$ for each valid $i,j$
	        \State Run algorithm 1 with the family $I_{i,j}$ above, which will yield an angle. Name it $\theta_{p,q}$.
	        
	\EndFor

\EndFor

\State $(P,Q) \gets \arg \min_{p,q} \left\lbrace E_{p,q}(\theta_{p,q}) \right\rbrace$
\State $\Theta \gets \theta_{P,Q}$

\end{algorithmic}
\end{algorithm}

Notice how algorithm 2 is entirely parallelizable: without making any changes, one may execute the first "for" loop in parallel. However, since the content of this loop is executed rather rapidly, this can result in significant overhead. Thus, a more efficient handling of this loop may be desired. In practice it was noted that parallelizing in this way with 4 cores only led to doubling the speed because of overhead.



\section{Experiments}

We verify the results and methods stated above by a series of numerical experiments, in a computer with 16 GB of RAM and a quad-core processor of 3.20 GHz per core. For all of them, gaussian noise shall be used, with  $\sigma = 0.01$, and the parameter $\delta$ shall be set to $0.01$ as well. 

Let us first illustrate with 2 concrete examples. For figure \ref{fig:sine}, the curve $y=\sin x$ sampled uniformly in $[0,2 \pi]$ by 200 points was rotated by 2 radians and then noise was added. This is the simplest scenario the algorithm is expected to handle: where $M=N$ and there are no outliers. We see how the point cloud registration was successful.

Now we propose a more challenging situation. For the first pointcloud $X$, the parametric curve $x=3\cos(t)$, $y=2\sin(t)$ is plotted over a uniform partition of 200 points in $[0,2\pi]$. However, in addition to rotating and adding noise to form the second pointcloud, several of the points are deleted, namely: those corresponding to the 51th to 69th, 111th to 169th, and 196th to 199th points of the partition. Then 50 random outliers, constructed from gaussian noise with $\sigma=2$, are added, and the result is the second pointcloud $Y$. Under these conditions, no pointcloud can be included in the other: both of them possess several outliers. The initialization, registration and common subset found are shown in figure \ref{fig:Ellipse}, and it is seen that the registration was successful.

\begin{figure}
    \centering
    \includegraphics[height=3.5cm]{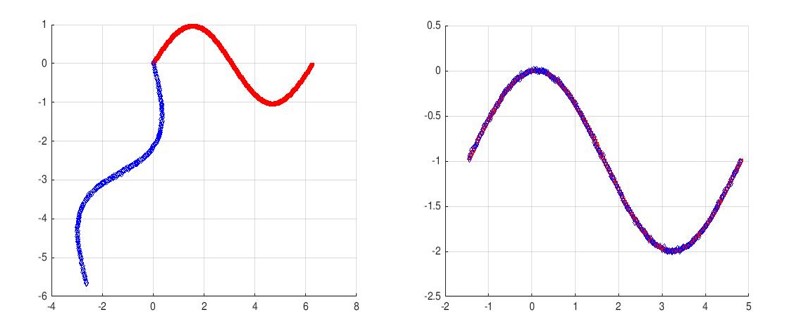}
    \caption{Left: initialization of the experiment, with a rotated version of the sinusoid with noise. Right: result of the pointcloud registration}
    \label{fig:sine}
\end{figure}

\begin{figure}
    \centering
    \includegraphics[height=3.5cm]{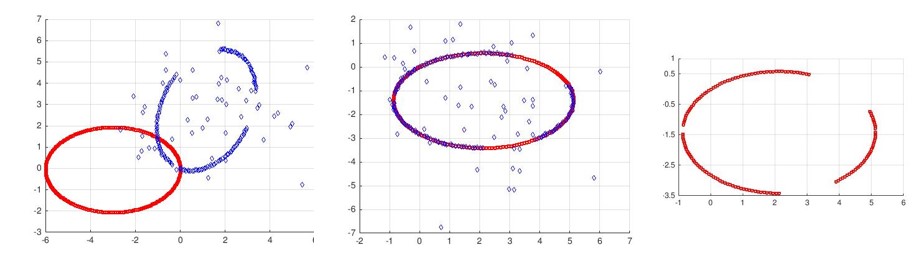}
    \caption{Left: initialization of the experiment. Center: result of the pointcloud registration. Right: maximal common subset found}
    \label{fig:Ellipse}
\end{figure}

Both of the registrations above took roughly a minute when ran under a single core, and 30 seconds if parallelized in the naive way described above. Now, after these 2 examples, we take on 2 experiments of a more quantitative nature:

\subsection{Experiment: Classification of subsets from a small pointcloud library:}

The idea for this experiment is to test the following: given a pointcloud library and a subset of one of the pointclouds which is rotated and then added noise, can the algorithm match the subset to the proper pointcloud consistently?

To analyze this, a collection of 50 pointclouds of 150 points each was created by taking random points from the unit disk. Then, a random cloud from the collection was taken, and random subset of 75 to 150 points was chosen from it. This subset was rotated by a random angle, and distorted by gaussian noise with $\sigma=0.01$. Our algorithm was run to find a maximal common subset for each pointcloud in the collection, and the pointcloud with which there is a maximum number of matches is chosen as the corresponding pointcloud.

This procedure was repeated 50 times. The subset was matched to the correct pointcloud every single time, with a difference in the angle of rotation found and the correct angle of much less than a degree. In fact, the 50 differences (in absolute value) analyzed had a mean of 0.0028 degrees, a standard deviation of 0.0023, and a maximum of 0.0080.

\subsection{Experiment: Matching in function of the number of points in the maximal common subset} 
This second experiment attempts to answer the following: under the conditions of a noisy match (for example, those stated at the beginning of this section, with $\delta = \sigma =0.01$), how many points in common do the 2 subclouds need to have in order to reliably do the pointcloud registration?

Had the algorithm been designed to find an exact match and work in situations with no noise, it would be expected that having 3 points in common would be enough to find the correct rigid motion that overlaps them (since then the rigid motion is uniquely determined, assuming no 2 distances between the 3 points are exactly the same). However, the existence of noise can make unrelated points align, as well as the use of noisy matches.

In order to test to what extent is the pointcloud registration affected by the size of the common subset, we take a random pointcloud of 300 points in the unit circle. We take 2 subsets $X$ and $Y$ of 150 points each from it, so that $X$ and $Y$ share $k$ points. Then $Y$ is rotated by a random angle and added the described noise. The algorithm is ran for representative values of $k$, and repeated 20 times for each of these values.

We will call each registration a success if the angle found differs from the correct angle by less than a degree. With this criterion, our results are summarized in Figure \ref{fig:plot}, which shows percentage of succesful matches in function of chosen size of the common subcloud. For $k>70$, the match was succesful every single time. As the graph shows, it is for $k \geq 30$ that the match can be done reliably under the aforementioned conditions. This amounts to both pointclouds having a $20 \%$ of points in common, or more.

\begin{figure}
    \centering
    \includegraphics[height=4cm]{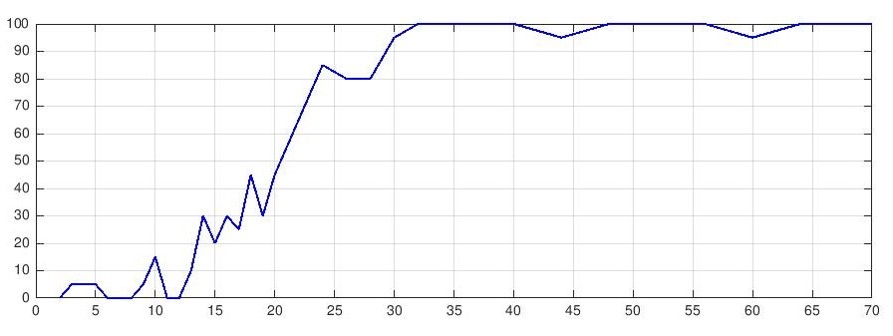}
    \caption{Percentage of succesful registrations as a function of the number of common points.}
    \label{fig:plot}
\end{figure}



\bibliography{biblio}
\bibliographystyle{abbrv}

\end{document}